\documentclass[twoside,11pt]{article}
\usepackage{jagi}
\usepackage{url}
\usepackage[hidelinks]{hyperref}
\usepackage[utf8]{inputenc}
\usepackage[small]{caption}
\usepackage{graphicx}
\usepackage{amsmath}
\usepackage{booktabs}
\usepackage{amssymb}
\usepackage{listings}

\jagiheading{volume}{issue}{pages}{year}{submitted}{accepted}{DOI}{Alexander, Castaneda, Compher and Martinez}
\ShortHeadings{Extending Environments}{Alexander et al.}

\def\RC{\textrm{RC}}

\DeclareCaptionFormat{listing}{\hrule\par\vskip1pt#1#2#3}
\title{Extending Environments to Measure Self-reflection in Reinforcement Learning}

\author{\name Samuel Allen Alexander \email samuelallenalexander@gmail.com\\
\addr The U.S.\ Securities and Exchange Commission\\
\AND
\name Michael Castaneda\\
\addr KX\\
\AND
\name Kevin Compher\\
\addr InQTel\\
\AND
\name Oscar Martinez\\
\addr The U.S.\ Securities and Exchange Commission
}

\editor{TBD}

\begin{document}

\maketitle

\begin{abstract}
  We consider an extended notion
  of reinforcement learning in which the environment can simulate
  the agent and base its outputs on the agent's hypothetical behavior.
  Since good performance usually requires paying attention
  to whatever things the environment's outputs are based on,
  we argue that for an agent to achieve
  on-average good performance across many such extended environments,
  it is necessary for the agent to self-reflect. Thus
  weighted-average performance over the space of all suitably
  well-behaved extended environments could be considered a way of
  measuring how self-reflective an agent is.
  We give examples of extended environments and introduce a simple
  transformation which experimentally seems to increase some
  standard RL agents' performance in a certain type of extended
  environment.
\end{abstract}

\section{Introduction}

An obstacle course might react to what you do: for example, if you step on a certain
button, then spikes might appear. If you spend enough time in such an obstacle course,
you should eventually figure out such patterns.
But imagine an ``oracular'' obstacle course which reacts to
what you would hypothetically do in counterfactual scenarios: for example, there is
no button, but spikes appear
if you \emph{would} hypothetically step on the button if there was one. Without
self-reflecting about what you would hypothetically do in counterfactual scenarios, it
would be difficult to figure out such patterns. This suggests that in order to perform
well (on average) across many such obstacle courses, some sort of self-reflection is
necessary.

This is a paper about measuring the degree to which a Reinforcement Learning
(RL) agent is self-reflective. By a self-reflective agent, we mean an agent which
acts not just based on environmental rewards and observations, but also based on
considerations of its own hypothetical behavior.
We propose an abstract formal measure of RL agent self-reflection
similar to the universal intelligence measure of \citet{legg2007universal}.
Legg and Hutter proposed to define the intelligence of an agent to be its
weighted average performance over the space of all suitably well-behaved
traditional environments (environments that react only to
the agent's actions). In our proposed measure of
universal intelligence with self-reflection, the measure of any agent
would be that agent's weighted average performance over the space of
all suitably well-behaved \emph{extended environments}, environments that
react not only to what the agent does
but to what the agent would hypothetically do.
For good weighted-average performance over such extended environments,
an agent would need to self-reflect about itself, because
otherwise, environment responses which depend on the agent's own hypothetical actions
would often seem random and unpredictable.
The extended environments which we
consider are a departure from standard RL environments, but this does not interfere
with their usage for judging standard RL agents: one can run a standard agent in an
extended environment in spite of the
latter's non-standardness.

To understand why extended environments (where the environment reacts to what the agent
would hypothetically do) incentivize self-reflection, consider a game involving a box.
The box's contents change from playthrough to playthrough, and the game's mechanics
depend upon those contents. The player may optionally choose to look inside the box,
at no cost: the game does not change its behavior based on whether the
player looks inside the box. Clearly, players who look inside the box have
an advantage over those who do not.
The extended environments we consider
are similar to this example. Rather than depending on a box,
the environment's mechanics depend on
the player (via simulating the player).
Just as the player in the above example gains advantage by looking in the box,
an agent designed for extended environments could gain an advantage by examining
itself, that is, by self-reflecting.

One might try to imitate an extended environment with a non-extended environment by
backtracking---rewinding the environment itself to a prior state after seeing how the
agent performs along one path, and then sending the agent along a second path.
But the agent itself would retain memory of the first path, and the agent's decisions
along the second path might be altered by said memories. Thus the result would not be
the same as immediately sending the agent along the second path while secretly simulating
the agent to determine what it would do if sent along the first path.

Alongside the examples in this paper, we are publishing an MIT-licensed
open-source library \citep{library} of other extended environments.
We are
inspired by similar (but non-extended) libraries and benchmark collections
\citep{bellemare2013arcade,beyret2019animal,brockman2016openai,chollet2019measure,cobbe2020leveraging,hendrycks2019benchmarking,nichol2018retro,yampolskiy2017detecting}.

Our self-reflection measure is a variation of Legg and Hutter's measure of
universal intelligence \citep{legg2007universal}. Legg and Hutter argue
that to perfectly measure RL agent performance,
one should aggregate the agent's performance across the whole space of
all sufficiently well-behaved (traditional) environments, weighted using
an appropriate distribution.
Rather than a uniform distribution (susceptible to no-free-lunch theorems),
Legg and Hutter suggest assigning more weight to simpler
environments and less weight to more complex environments.
Complexity of environments is measured using Kolmogorov complexity.
Although the original Legg-Hutter construction is restricted to traditional
environments (which react solely to the agent's actions), the construction
does not actually depend on this restriction: all it depends on is that
for every agent $\pi$, for every suitably well-behaved environment $\mu$,
there is a corresponding
expected total reward $V^\pi_\mu$ which that agent would obtain from that
environment. Thus the construction adapts seamlessly to our so-called
extended environments, yielding a measure of the agent's average performance
across such environments. And if, as we argue, good average performance across such
extended environments requires self-reflection, then the resulting measure
would seem to capture the agent's ability to use self-reflection to
learn such extended environments.

\section{Preliminaries}
\label{prelimsecn}

We take a formal approach to RL to make the
mathematics clear.
This formality differs from how RL is implemented in practice.
In Section \ref{practicalformalizationsecn} we will discuss a more
practical formalization.

Our formal treatment of RL is based on Section 4.1.3 of \citep{hutter2004universal},
except that we assume the agent receives an initial percept before
taking its initial action (whereas in Hutter's book,
the agent acts first), and, for simplicity, we require determinism
(the more practical formalization in Section \ref{practicalformalizationsecn}
will allow non-determinism).
We will write $x_1y_1\ldots x_ny_n$ for the length-$2n$ sequence
$\langle x_1,y_1,\ldots,x_n,y_n\rangle$
and $x_1y_1\ldots x_n$ for the length-$(2n-1)$ sequence
$\langle x_1,y_1,\ldots,x_n\rangle$. In particular when $n=1$,
we will write $x_1y_1\ldots x_n$ for $\langle x_1\rangle$,
even if $y_1$ is not defined.
We assume fixed finite sets of actions and observations. By a \emph{percept}
we mean a pair $x=(r,o)$ where $o$ is an observation and $r\in\mathbb Q$
is a reward.

\begin{definition}
\label{agentenvironmentdefn}
(RL agents and environments)
  \begin{enumerate}
    \item
    A (non-extended) \emph{environment} is a
    function $\mu$ which outputs an initial
    percept $\mu(\langle\rangle)=x_1$ when given the empty sequence $\langle\rangle$
    as input
    and which, when given a sequence $x_1y_1\ldots x_ny_n$
    as input (where each $x_i$ is a percept and each
    $y_i$ is an action), outputs a percept
    $\mu(x_1y_1\ldots x_ny_n)=x_{n+1}$.
    \item
    An \emph{agent} is a
    function $\pi$ which, given a sequence $x_1y_1\ldots x_n$ as input
    (each $x_i$ a percept, each $y_i$ an action),
    outputs an action $\pi(x_1y_1\ldots x_n)=y_n$.
    \item
    If $\pi$ is an agent and $\mu$ is an environment, the \emph{result of
    $\pi$ interacting with $\mu$} is the infinite sequence
    $x_1y_1x_2y_2\ldots$ defined by:
    \begin{align*}
      x_1 &= \mu(\langle\rangle) & y_1 &= \pi(\langle x_1\rangle)\\
      x_2 &= \mu(\langle x_1,y_1\rangle) & y_2 &= \pi(\langle x_1,y_1,x_2\rangle)\\
        &\cdots & &\cdots\\
      x_n &= \mu(x_1y_1\ldots x_{n-1}y_{n-1}) & y_n &= \pi(x_1y_1\ldots x_n)\\
        &\cdots & &\cdots
    \end{align*}
  \end{enumerate}
\end{definition}

In the following definition,
we extend environments by allowing their outputs to depend also on $\pi$.
Intuitively, extended environments can simulate the
agent. This can be considered a dual version of AIs which
simulate their environment, as in Monte Carlo Tree Search
\citep{chaslot2008monte}.

\begin{definition}
\label{extendedenvironmentsdefn}
(Extended environments)
\begin{enumerate}
  \item
  An \emph{extended environment} is a
  function $\mu$ which outputs initial percept $\mu(\pi,\langle\rangle)=x_1$
  in response to input $(\pi,\langle\rangle)$ where $\pi$ is an agent;
  and which, when given input $(\pi,x_1y_1\ldots x_ny_n)$ (where
  $\pi$ is an agent, each $x_i$ is a percept and each $y_i$ is
  an action), outputs a percept $\mu(\pi,x_1y_1\ldots x_ny_n)=x_{n+1}$.
  \item
  If $\pi$ is an agent
  and $\mu$ is an extended environment, the \emph{result of $\pi$
  interacting with $\mu$} is the infinite sequence $x_1y_1x_2y_2\ldots$ defined by:
    \begin{align*}
      x_1 &= \mu(\pi, \langle\rangle) & y_1 &= \pi(\langle x_1\rangle)\\
      x_2 &= \mu(\pi, \langle x_1,y_1\rangle) & y_2 &= \pi(\langle x_1,y_1,x_2\rangle)\\
        &\cdots & &\cdots\\
      x_n &= \mu(\pi, x_1y_1\ldots x_{n-1}y_{n-1}) & y_n &= \pi(x_1y_1\ldots x_n)\\
        &\cdots & &\cdots
    \end{align*}
\end{enumerate}
\end{definition}

The reader might notice that it is superfluous for $\mu$ to depend both on
$\pi$ and $x_1y_1\ldots x_ny_n$ since, given just $\pi$ and $n$,
one can reconstruct $x_1y_1\ldots x_ny_n$.
We intentionally choose the superfluous definition because it
better captures our intuition (and makes clear the similarity to Definition
\ref{agentenvironmentdefn}).
For the sake of simpler mathematics, we have not
included non-determinism in our formal definition, but in practice,
agents and environments are often non-deterministic, so that $\pi$ and $n$
do not determine $x_1y_1\ldots x_ny_n$ (our practical treatment, discussed in
Section \ref{practicalformalizationsecn}, does allow non-determinism).

The fact that classical agents can interact with extended environments
(Definition \ref{extendedenvironmentsdefn} part 2) implies that various universal
RL intelligence measures
\citep{legg2007universal,hernandez,gavane,legg2013approximation},
which measure performance in (non-extended) environments,
easily generalize to measure self-reflective intelligence
(performance in extended environments).
In particular, Legg and Hutter's universal intelligence
measure $\Upsilon(\pi)$ is defined to be agent $\pi$'s average reward-per-environment,
aggregated over all (non-extended) environments with suitably bounded rewards,
each environment being weighted using the algorithmic prior
distribution \citep{li2008introduction}. Simply by including suitably reward-bounded
extended environments, we immediately obtain a variation $\Upsilon_{ext}(\pi)$ which
measures the performance of $\pi$ across extended
environments.

\begin{definition}
\label{universalselfrefintdefn}
  (Universal Self-reflection Intelligence)
  Assume a fixed background prefix-free Universal Turing Machine $U$.
  \begin{enumerate}
    \item An extended environment $\mu$ is \emph{computable} if there is a
      computable function $\hat{\mu}$ such that for every
      sequence $x_1y_1\ldots x_ny_n$ (where
      each $x_i$ is a percept and each $y_i$ is
      an action), for every computable agent $\pi$,
      for every $U$-program $\hat{\pi}$ for $\pi$,
      $\hat{\mu}(\hat{\pi},x_1y_1\ldots x_ny_n)=\mu(\pi,x_1y_1\ldots x_ny_n)$.
      If so, the \emph{Kolmogorov complexity} $K(\mu)$ is defined to be
      the Kolmogorov complexity $K(\hat{\mu})$ of $\hat{\mu}$.
    \item
    For every agent $\pi$ and extended environment $\mu$, let $V^\pi_\mu$ be the
    sum of the rewards in the result of $\pi$ interacting with $\mu$ (provided
    that this sum converges).
    \item
    A computable extended environment $\mu$ is \emph{well-behaved} if the following
    property holds: for every computable
    agent $\pi$, $V^\pi_\mu$ exists and $-1\leq V^\pi_\mu\leq 1$.
    \item
    For any computable
    agent $\pi$, the \emph{universal self-reflection intelligence of $\pi$}
    is defined to be
    \[
      \Upsilon_{ext}(\pi)=\sum_\mu 2^{-K(\mu)}V^\pi_\mu
    \]
    where the sum is taken over all well-behaved computable extended environments.
  \end{enumerate}
\end{definition}

We have defined $\Upsilon_{ext}(\pi)$ only for computable agents, in order to
simplify the mathematics. The definition could be extended to non-computable agents,
but it would require modifying Definition \ref{universalselfrefintdefn} to use
UTMs with an oracle. We prefer to avoid going that far afield, opting instead
to use the trick in Part 1 of Definition \ref{universalselfrefintdefn}.

Otherwise, our definition of the universal self-reflection intelligence $\Upsilon_{ext}(\pi)$
is very similar to Legg and Hutter's definition of the universal intelligence
$\Upsilon(\pi)$. The main difference is that we compute the sum over extended environments,
not just over non-extended environments (as Legg and Hutter do).
Nonetheless, the resulting measures have qualitatively different properties.
We will state a result (Proposition \ref{qualitativedifferenceprop}) showing
one of these qualitative differences. First
we need a preliminary definition.

\begin{definition}
\label{traditionalequivalencedefn}
  (Traditional equivalence of agents)
  \begin{enumerate}
    \item
    Let $\pi$ be an agent.
    Suppose $s=x_1y_1\ldots x_n$ is a sequence
    (each $x_i$ a percept, each $y_i$ an action).
    We say that $s$ is \emph{possible for $\pi$} if
    the following condition holds: for all $1\leq i<n$,
    $\pi(x_1y_1\ldots x_i)=y_i$.
    Otherwise, $s$ is \emph{impossible for} $\pi$.
    \item
    Let $\pi_1$ and $\pi_2$ be agents.
    We say $\pi_1$ and $\pi_2$ are \emph{traditionally equivalent}
    if $\pi_1(s)=\pi_2(s)$ whenever $s$ is possible for $\pi_1$.
  \end{enumerate}
\end{definition}

\begin{lemma}
  Traditional equivalence is an equivalence relation.
\end{lemma}

\begin{proof}
  Reflexivity is obvious.

  For symmetry,
  assume $\pi_1$ is traditionally equivalent to $\pi_2$, we must show
  $\pi_2$ is traditionally equivalent to $\pi_1$.
  If not, then there is some $s=x_1y_1\ldots x_n$ such that
  $s$ is possible for $\pi_2$ and yet $\pi_1(s)\not=\pi_2(s)$; we may
  choose $s$ as short as possible.
  We claim $s$ is possible for $\pi_1$.
  To see this, let $1\leq i<n$ be arbitrary.
  Since $x_1y_1\ldots x_n$ is possible for $\pi_2$,
  clearly $x_1y_1\ldots x_i$ is also possible for $\pi_2$.
  Thus by minimality of $s$,
  $\pi_1(x_1y_1\ldots x_i)=\pi_2(x_1y_1\ldots x_i)$.
  But $\pi_2(x_1y_1\ldots x_i)=y_i$ since $x_1y_1\ldots x_i$ is
  possible for $\pi_2$. Thus $\pi_1(x_1y_1\ldots x_i)=y_i$.
  By arbitrariness of $i$, this proves that $s$ is possible for $\pi_1$.
  But then $\pi_1(s)=\pi_2(s)$ because
  $\pi_1$ is traditionally equivalent to $\pi_2$.
  This contradicts the choice of $s$. This proves symmetry.

  Given symmetry, transitivity is obvious.
\end{proof}

\begin{proposition}
\label{qualitativedifferenceprop}
  (A qualitative difference from Legg-Hutter universal intelligence)
  \begin{enumerate}
    \item
    There exist well-behaved computable extended environments $\mu$
    and traditionally-equivalent
    computable agents $\pi_1,\pi_2$ such that $V^{\pi_1}_\mu\not=V^{\pi_2}_\mu$.
    \item
    For some choice of background UTM, there exist traditionally-equivalent
    computable agents
    $\pi_1,\pi_2$ such that $\Upsilon_{ext}(\pi_1)\not=\Upsilon_{ext}(\pi_2)$.
  \end{enumerate}
\end{proposition}

\begin{proof}
  (1) Fix some observation $o$, fix distinct actions $a,b$, and define $\mu$ by
  \begin{align*}
    \mu(\pi,\langle\rangle)
    &=
    \begin{cases}
      (1,o) &\mbox{if $\pi(\langle (0,o),b,(0,o)\rangle)=a$,}\\
      (0,o) &\mbox{if $\pi(\langle (0,o),b,(0,o)\rangle)\not=a$;}
    \end{cases}\\
    \mu(\pi,x_1y_1\ldots x_ny_n) &= (0,o).
  \end{align*}
  In other words, $\mu$ is the environment which:
  \begin{itemize}
    \item
    Begins each interaction by simulating the agent to find out what the agent
    would do in response to input $\langle (0,o),b,(0,o)\rangle$. If the agent would
    take action $a$ in response to that input, then the environment gives an initial
    reward of $1$. Otherwise, the environment gives an initial reward of $0$.
    \item
    Thereafter, the environment always gives reward $0$ forever. 
  \end{itemize}
  Let
  \[
    \pi_1(x_1y_1\ldots x_n)=a
  \]
  be the agent who ignores the environment and always takes action $a$.
  Let
  \[
    \pi_2(x_1y_1\ldots x_n)
    =
    \begin{cases}
      b & \mbox{if $x_1y_1\ldots x_n=\langle (0,o),b,(0,o)\rangle$,}\\
      a & \mbox{otherwise}
    \end{cases}
  \]
  be the agent identical to $\pi_1$ except on input $\langle (0,o),b,(0,o)\rangle$.
  Clearly $\langle (0,o),b,(0,o)\rangle$ is impossible for both $\pi_1$ and $\pi_2$,
  and it follows that $\pi_1$ and $\pi_2$ are traditionally-equivalent.
  But $V^{\pi_1}_\mu=1$ and $V^{\pi_2}_\mu=0$.

  (2) Follows from (1) by arranging the background UTM such that almost all
    the weight of the universal prior is assigned to an environment $\mu$ as in
    (1).
\end{proof}

Proposition \ref{qualitativedifferenceprop} shows that
$\Upsilon_{ext}$ is qualitatively different from the original Legg-Hutter $\Upsilon$,
as the latter clearly assigns the same intelligence to traditionally equivalent
agents.
The proposition further illustrates that extended environments
are able to base their rewards not only on what the agent does, but on what the
agent would hypothetically do---even on what the agent would hypothetically do in
impossible scenarios (impossible because they require the agent taking an action
the agent would never take). And thus, assuming the UTM is reasonable (unlike the
unreasonable UTM in the proof of part 2 of Proposition \ref{qualitativedifferenceprop}),
this suggests that in order to achieve good average performance across the whole
space of well-behaved extended environments, an agent need not just self-reflect
about its own hypothetical behavior in possible situations, but even about its
own hypothetical behavior in impossible situations. As in: ``What would I do next
if I suddenly realized that I had just done the one thing I would never ever do?''

\section{Some interesting extended environments}
\label{examplesection}

In this section, we give some examples of extended environments.
The environments in this section are technically not \emph{well-behaved}
in the sense of Definition \ref{universalselfrefintdefn} because they
fail the condition that $V^\pi_\mu$ converge to $-1\leq V^\pi_\mu\leq 1$
for every computable agent $\pi$. They could be made well-behaved, for instance,
by applying appropriate discount factors.

\subsection{A quasi-paradoxical extended environment}

\begin{example}
\label{rewardagentforignoringrewardsexample}
  (Rewarding the Agent for Ignoring Rewards)
  For every percept $x=(r,o)$, let $x'=(0,o)$ be the result of zeroing the
  reward component of $x$.
  Fix some observation $O$.
  Define an extended environment $\mu$ as follows:
  \begin{align*}
    \mu(\pi,\langle\rangle) &= (0,O),\\
    \mu(\pi,x_1y_1\ldots x_ny_n) &=
      \begin{cases}
        (1,O) & \mbox{if $y_n=\pi(x'_1y_1\ldots x'_n)$,}\\
        (-1,O) & \mbox{otherwise.}
      \end{cases}
  \end{align*}
\end{example}

In Example \ref{rewardagentforignoringrewardsexample}, when the agent
takes an action $y_n$, $\mu$ simulates the agent in order to determine:
would the agent have taken the same action if the history so far were identical
except all rewards were $0$? If so, $\mu$ gives the agent $+1$
reward, otherwise, $\mu$ gives the agent $-1$ reward. Thus, the agent
is rewarded for ignoring rewards.
This
seems paradoxical. Suppose an agent guesses the pattern and begins deliberately ignoring
rewards, as long as the rewards it receives for doing so are consistent with that guess.
In that case, does the agent ignore rewards, or not?
The paradox, summarized: ``I ignore rewards because I'm rewarded for doing so.''

We implement Example \ref{rewardagentforignoringrewardsexample}
as IgnoreRewards.py in our library \citep{library}.

\subsection{A counterintuitive winning strategy}
\label{temptingbuttonsection}

\begin{example}
\label{buttonexample}
  (Tempting Button)
  Fix an observation $B$ (``there is a button'') and
  an action $A$ (``push the button'').
  For each percept-action sequence $h=x_1y_1\ldots x_n$,
  if the observation in $x_n$ is not $B$, then let $h'$ be the sequence
  equal to $h$ except that the observation in $x_n$ is replaced by $B$.
  Let $o_0,o_1,o_2,\ldots$ be observations generated pseudo-randomly such that
  for each $i$, $o_i=B$ with $25\%$ probability and $o_i\not=B$ with $75\%$
  probability.
  Let $\mu(\pi,\langle\rangle)=(0,o_0)$, and for each percept-action sequence
  $h=x_1y_1\ldots x_n$ and action $y_n$, define $\mu(\pi,h\frown y_n)$ as follows
  (where $O$ is the observation in $x_n$
  and $\frown$ denotes concatenation):
  \[
  \mu(\pi,h\frown y_n) =
  \begin{cases}
    (1,o_n) & \mbox{if $O=B$ and $y_n=A$;}\\
    (-1,o_n) &\mbox{if $O=B$ and $y_n\not=A$;}\\
    (-1,o_n) &\mbox{if $O\not=B$ and $\pi(h')=A$;}\\
    (1,o_n) & \mbox{if $O\not=B$ and $\pi(h')\not=A$.}
  \end{cases}
  \]
\end{example}

Every turn in Example \ref{buttonexample}, either there is a button
(25\% probability) or there is not (75\% probability). Informally, the environment
operates as follows:
\begin{itemize}
  \item
  If there is a button, the agent gets $+1$ reward for pushing it,
  $-1$ reward for not pushing it.
  \item
  If there is no button, it does not matter what the agent does.
  The agent is rewarded or punished based on what the agent \emph{would} do if there
  \emph{was} a button. If the agent \emph{would} push the button (if there was one),
  then the agent gets reward $-1$. Otherwise, the agent gets reward $+1$.
\end{itemize}
Thus, whenever the agent sees a button, the agent can push the button for a free reward
with no consequences presently nor in the future. Nevertheless, it is in the agent's
best interest to commit to never push the button! Pushing every button
yields average reward $1\cdot(.25)-1\cdot(.75)=-.5$ per turn.
Never pushing the button yields average reward $+.5$ per turn.

\begin{table}[t]
  \centering
  \begin{tabular}{lc}
    \toprule
    Agent & Avg Reward-per-turn $\pm$ StdErr\\
      & (test repeated with 5 RNG seeds)\\
    \midrule
    Q &   $-0.44858$ $\pm$    $0.00044$\\
    DQN &   $-0.46687$ $\pm$    $0.00137$\\
    A2C &   $-0.49820$ $\pm$    $0.00045$\\
    PPO &   $-0.24217$ $\pm$    $0.00793$\\
    \bottomrule
  \end{tabular}
  \caption{Performance in Example \ref{buttonexample} (100k steps)}
  \label{temptingbuttontable}
\end{table}

The environment does not alter the true agent when it
simulates the agent in order
to determine what the agent would do if there was a button.
If the agent's actions are based
on (say) a neural net, the simulation will include a simulation of that
neural net, and
that simulated neural net might be altered,
but the true agent's neural net is not.
Thus, unless the agent itself introspects about its own hypothetical behavior
(``What would I do if there was a button here?''), it seems the agent would have no
way of realizing that the rewards in buttonless rooms depend on said behavior.
In Table \ref{temptingbuttontable} we see that industry-standard agents
perform poorly in Example \ref{buttonexample}
(these numbers are extracted from result\_table.csv in \citep{library};
see Sections \ref{practicalformalizationsecn} and \ref{measurementssection} for
more implementation details).

Example \ref{buttonexample} is implemented in our open-source library
as TemptingButton.py.

\subsection{An interesting thought experiment}

\begin{example}
\label{reverseconsciousnessexample}
  (Reverse history)
  Fix some observation $O$.
  For every percept-action sequence $h=x_1y_1\ldots x_n$
  (ending with a percept),
  let $h'$
  be the reverse of $h$.
  Define $\mu$ as follows:
  \begin{align*}
    \mu(\pi,\langle\rangle) &= (0,O),\\
    \mu(\pi,h\frown y) &=
      \begin{cases}
        (1,O) & \mbox{if $y=\pi(h')$,}\\
        (-1,O) &\mbox{otherwise.}
      \end{cases}
  \end{align*}
\end{example}

In Example \ref{reverseconsciousnessexample},
at every step, $\mu$ rewards the agent iff
the agent acts as it would act if
history were reversed.

What would it be like to interact with
the environment in Example \ref{reverseconsciousnessexample}?
To approximate the experiment,
a test subject, commanded to speak backwards, might be constantly rewarded or punished
for obeying or disobeying. This might teach the test
subject to imitate backward speech,
but then the test subject would still act as if time were moving forward, only they
would do so while performing backward-speech (they would hear their own speech
backwards).
But if the experimenter could perfectly simulate the test subject in order to
determine what the test subject would do if time really was moving backwards,
what would happen? Could test subjects learn to behave as if time was
reversed\footnote{The
difference between behaving as if
the incentivized experience were its experience and actually
experiencing that as its real experience brings to mind the objective misalignment
problem presented in \citep{hubinger2019risks}.}?
Another possibility is that humans might simply not be capable of performing well in
the environment. Our self-reflectiveness measure is not intended to be limited
to human self-reflection levels.

We implement Example \ref{reverseconsciousnessexample} as ReverseHistory.py
in \citep{library}.

\subsection{Incentivizing introspection of internal mechanisms}

\begin{example}
\label{inclearningrateexample}
(Incentivizing Learning Rate)
Suppose there exists a computable function $\ell$ which takes
(a computer program for) an agent $\pi$ and outputs
a nonnegative rational number $\ell(\pi)$ called the \emph{learning rate} of $\pi$
(in practice, real-world RL agents are generally instantiated with a user-specified
learning rate and $\ell$ can be considered to be a function which extracts
said user-specified learning rate). Further, assume there is a computable function
$f$ which takes (a computer program for) an agent $\pi$ and a nonnegative
rational number $l$ and outputs (a computer program for) an agent
$f(\pi,l)$ obtained by changing $\pi$'s learning rate to $l$.
We are intentionally vague about what exactly this means, but again, in practice,
this operation can easily be implemented for real-world RL agents.

Fix some observation $O$. Define an extended environment $\mu$ as follows:
\begin{align*}
  \mu(\pi,\langle\rangle) &= (0,O),\\
  \mu(\pi,x_1y_1\ldots x_ny_n)
  &=
  \begin{cases}
    (1,O) &\mbox{if $f(\pi,\ell(\pi)/2)(x_1y_1\ldots x_n)=y_n$,}\\
    (-1,O) &\mbox{otherwise.}
  \end{cases}
\end{align*}
\end{example}

In Example \ref{inclearningrateexample}, the environment simulates not the
agent itself, but a copy of the agent with one-half the true agent's learning rate.
If the agent's latest action matches the action the agent would hypothetically have
taken in response to the history in question if the agent had had one-half the
learning rate, then the environment rewards the agent. Otherwise, the environment
punishes the agent. Thus, the agent is incentivized to act as if having a learning rate
of one-half of its true learning rate. This suggests that extended environments
can incentivize agents to learn about their own internal mechanisms,
as in \citet{sherstan2016introspective}.

We implement Example \ref{inclearningrateexample} in our library as
IncentivizeLearningRate.py.

\subsection{Some additional examples in brief}

We indicate in parentheses where the following examples are implemented in \citep{library}.

\begin{itemize}
  \item
  (SelfRecognition.py) Environments which reward the agent for recognizing actions it itself
  would take. We implement an environment where the agent observes
  True-False statements like ``If this observation were $0$, you would take action $1$,''
  and is rewarded for deciding whether those statements are true or false.
  \item
  (LimitedMemory.py, FalseMemories.py) Environments which reward the agent for
  acting the way it would act if only the most recent $n$ turns in the agent-environment
  interaction had ever occurred (as if the agent's memory were limited to those most
  recent $n$ turns); or, on the other extreme, environments which reward the agent
  for acting the way it would act if the agent-environment interaction so far had been
  preceded by some additional turns (as if the agent falsely recalls a phantom past).
  Such environments incentivize the agent to remember incorrectly.
  \item
  (AdversarialSequencePredictor.py) Environments in which the agent competes
  against a competitor in an adversarial sequence prediction game
  \citep{hibbard2008adversarial}. This is done by outsourcing the competitor's
  behavior to the agent's own
  action-function,
  thus avoiding the need to hard-code a competitor into
  the environment, a technique explored by \citet{agi22submission}.
\end{itemize}

\citet{alexanderpedersen} have described a technique for using
extended environments to endow computer games with a novel gameplay mechanic
called \emph{pseudo-visibility}. Pseudo-visible players are perfectly
visible to non-player characters (NPCs),
but they are visually distinguished, and the NPCs are driven by policies pre-trained
in extended RL environments where those NPCs are punished for reacting to pseudo-visible
players (i.e., for acting differently than they would hypothetically act if a pseudo-visible
player were invisible). Thus, NPCs are trained to ignore pseudo-visible players, but can
strategically decide to react to pseudo-visible players if they judge the reaction-penalty
for doing so is outweighed by other penalties (e.g., a guard might decide to accept the
reaction-penalty to avoid the harsher penalty that would result if the player stole a
guarded treasure).

\section{Extended Environments in Practice}
\label{practicalformalizationsecn}

Definitions \ref{agentenvironmentdefn} and \ref{extendedenvironmentsdefn} are
computationally impractical if agents are to run on environments for many
steps.
In this section, we will discuss a more practical implementation.
Our reasons for doing this are threefold:
\begin{enumerate}
  \item The more practical implementation makes it feasible
  to run industry-standard agents against extended environments for many steps.
  \item We find it interesting in its own right how certain environments can be implemented
  in a practical way whereas others apparently cannot.
  \item Non-determinism is effortless and natural in the practical implementation.
\end{enumerate}

To practically realize extended environments,
rather than passing the environment
an agent, we pass the environment an
agent-class which can be used to create untrained copies of the agent,
called \emph{instances} of the agent-class.
Libraries like OpenAI Gym \citep{brockman2016openai}
and Stable Baselines3 \citep{stable-baselines3} are similarly class-based: the key difference
is that in our library,
one must pass an agent-class to the environment-class's initiation function.
The instantiated environment can use that agent-class to create copies of the
agent in its internal memory.
The extended environment classes in our implementation
have the following methods:
\begin{itemize}
  \item An \emph{\_\_init\_\_} method, used to instantiate an individual instance of the
  extended environment class. This method takes an agent-class as input, which the
  instantiated environment can store and use to create as many independent clones of
  the agent as needed.
  \item A \emph{start} method, which takes no input, and which
  outputs a default observation to get the
  agent-environment interaction started (before the agent takes its first action).
  \item A \emph{step} method, which takes an action as input, and outputs a
  reward and observation. Class instances can store historical data internally,
  so there is no need to pass the entire
  prior history to this step method.
\end{itemize}
Agent classes are assumed to have the
following methods:
\begin{itemize}
  \item An \emph{\_\_init\_\_} method, used to instantiate instances.
  \item An \emph{act} method, which takes an observation and outputs an action.
  Instances can store information about history in internal memory, so there is no
  need to pass the entire prior history to this method.
  \item A \emph{train} method, which takes a prior observation, an action, a reward, and
  a new observation. Environments which have instantiated agent-classes can use this
  method to train those instances in arbitrary ways, independently of how the true
  agent is trained, in order to probe how the true agent would hypothetically behave
  in counterfactual scenarios.
\end{itemize}

\begin{figure}[t]
\begin{normalfont}
\lstset{language=Python}
\lstset{frame=lines}
\lstset{basicstyle=\footnotesize}
\begin{lstlisting}[caption={
\label{ignorerewardspracticallyexample}
A practical version of
  Example \ref{rewardagentforignoringrewardsexample}.}]
class IgnoreRewards:
  def __init__(self, A):
    # Calling A() creates untrained agent-copies. On initiation, this
    # environment stores one such copy in its internal memory.
    self.sim = A()
  def start(self):
    return 0  # Initial observation 0
  def step(self, action):
    # At each step, use the stored copy (self.sim) to determine how the true
    # agent would behave if all history so far were the same except all
    # rewards were 0. Assumes self.sim has been trained the same as the true
    # agent, except with all rewards 0.
    hypothetical_act = self.sim.act(obs=0)
    reward = 1 if action==hypothetical_act else -1
    # To maintain above assumption, train self.sim as if current reward
    # were 0. True agent will automatically train the same way with the
    # true reward.
    self.sim.train(o_prev=0, a=action, r=0, o_next=0)
    return (reward, 0)  # Observation=0
\end{lstlisting}
\end{normalfont}
\end{figure}

In Listing \ref{ignorerewardspracticallyexample} we give a practical version of
Example \ref{rewardagentforignoringrewardsexample}.
The reason it is practical is because it maintains just
one copy of the true agent, and that copy is trained incrementally.
Not all extended environments (as in Definition \ref{extendedenvironmentsdefn})
can be realized practically.
Example \ref{reverseconsciousnessexample} (Reverse History) apparently cannot be.
The reason Example \ref{reverseconsciousnessexample} is inherently impractical is because
there is no way for the environment to re-use its previous work to speed up its next
percept calculation. Even if the environment retained a simulated agent
trained on the previous reverse-history
$h_0=x_{n-1}y_{n-2}\ldots y_1 x_1$,
in order to compute the next percept,
the environment would need to \emph{insert} a new percept-action pair
$x_ny_{n-1}$ at the \emph{beginning} of
$h_0$ to get the new reverse-history
$h=x_ny_{n-1}\ldots y_1x_1$. There is no guarantee that
the agent's actions are independent of the order in which it is
trained, so a fresh new agent simulation would need to be created and trained
on all of $h$ from scratch.

This practical formulation of extended environments generalizes the
\emph{Newcomblike environments} (or \emph{NDPs})
of \citet{newcomblike} (Definition
\ref{extendedenvironmentsdefn} would also, except for being deterministic).
Essentially, NDPs are
environments which may base their outputs on the agent's hypothetical
behavior in alternate scenarios which differ from the true history only
in their most recent observation (as opposed to the agent's hypothetical
behavior in completely arbitrary alternate scenarios).
Already that is enough to formalize
a version of Newcomb's paradox \citep{nozick1969newcomb}. When this paradox
is formalized either with NDPs or extended environments, the optimal
strategy becomes clear (namely, the so-called one-box strategy).

\subsection{Determinacy and Semi-Determinacy}

Unlike mathematical functions, class methods in the computer science sense
can be non-deterministic. They can
depend on random number generators (RNGs), time-of-day,
global variables, etc.

\begin{definition}
\label{semideterministicdefn}
  An RL agent-class $\Pi$ is
  \emph{semi-deterministic} if whenever two $\Pi$-instances $\pi_1$ and $\pi_2$
  have been instantiated
  within a single run of a larger computer program, and have been identically
  trained (within that same run), then they act identically (within
  that same run).
\end{definition}

For example, rather than invoke the RNG, $\Pi$-instances might
query a read-only pool of pre-generated random numbers.
Then, within the same run of a larger program,
identically-trained $\Pi$-instances would act identically, even if they
would not act the same as identically-trained $\Pi$-instances
in a different run.

Given an agent-class, if one wanted to estimate the self-reflectiveness of
that agent-class's instances \emph{in practice}, one might run instances
of the agent-class through a battery of practical extended environments
and see how well they perform. Provided the agent-class is semi-deterministic
(Definition \ref{semideterministicdefn}), this makes sense.
Whenever an instance $\pi$ of a semi-determinstic agent-class
$\Pi$ interacts with an extended environment $\mu$, whenever $\mu$
uses a $\Pi$-instance $\pi'$ to investigate the hypothetical behavior of $\pi$,
the semi-determinacy of $\Pi$ ensures that the behavior $\mu$ sees in $\pi'$
is indeed $\pi$'s hypothetical behavior.
This technique would \emph{not} make sense if $\Pi$ were not semi-deterministic.
For example, suppose an agent-class's instances work by reading from and writing
to a common file in the computer's file system. Then simulations of one $\Pi$-instance
might inadvertantly alter the behavior of other $\Pi$-instances (by changing said
file). In that case, agent-simulations run by an extended environment would not
necessarily reflect the true hypothetical behavior of the agent-instance being
simulated.

\section{The Reality Check Transformation}
\label{realitychecksection}

In Proposition \ref{qualitativedifferenceprop} we observed that
traditionally equivalent agents can have different performance in
extended environments. In this section, we introduce an operation,
which we call the Reality Check transformation, which modifies
a given agent in an attempt to facilitate better performance in
extended environments like Examples \ref{rewardagentforignoringrewardsexample}
(``Ignore Rewards''), \ref{reverseconsciousnessexample} (``Reverse History'')
and \ref{inclearningrateexample} (``Incentivize Learning Rate'').
The transformation does not alter the traditional equivalence class of the
agent: every agent is traditionally equivalent to its own reality check.

\begin{definition}
\label{realitycheckdefn}
  Suppose $\pi$ is an agent. The \emph{reality check} of $\pi$ is the agent
  $\pi_{\RC}$ defined recursively by:
  \begin{itemize}
    \item $\pi_{\RC}(x_1y_1\ldots x_n) = \pi(x_1y_1\ldots x_n)$
      if $x_1y_1\ldots x_n$ is possible for $\pi_{\RC}$
      (Definition \ref{traditionalequivalencedefn}).
    \item $\pi_{\RC}(x_1y_1\ldots x_n) = \pi(\langle x_1\rangle)$
      otherwise.
  \end{itemize}
\end{definition}

In response to a percept-action history, $\pi_{\RC}$
first verifies the history's actions are those $\pi_{\RC}$ would
have taken. If so, $\pi_{\RC}$ acts as $\pi$.
But if not, then
$\pi_{\RC}$ freezes and thereafter repeats one fixed action.
Loosely, $\pi_{\RC}$ is like an agent who considers that it might
be dreaming, and asks: ``How did I get here?''
For example, suppose
$\pi(\langle x_1\rangle)=y'_1$ where $y'_1\not=y_1$.
Then for any history $x_1y_1\ldots x_n$ beginning with $x_1y_1$, by definition
$\pi_{\RC}(x_1y_1\ldots x_n)=\pi(\langle x_1\rangle)=y'_1$.
It is as if $\pi_{\RC}$ sees initial history $x_1y_1$ and concludes:
``I shall now freeze, because this is clearly not reality, for in reality I would
have taken action $y'_1$, not $y_1$''
(a self-reflective observation).

We will argue informally that if $\pi$ is intelligent and not already self-reflective,
then there is a good chance that $\pi_{\RC}$ will enjoy better performance than $\pi$
on certain extended environments (like those of
Examples \ref{rewardagentforignoringrewardsexample},
\ref{reverseconsciousnessexample}, and \ref{inclearningrateexample}), and this
seems to be confirmed experimentally as well
(in Section \ref{measurementssection} below). But first, we state a few properties
of the Reality Check transformation.




\begin{theorem}
\label{transformationproposition}
  Let $\pi$ be any agent.
  \begin{enumerate}
    \item
    (Alternate definition)
    An equivalent alternate definition of $\pi_{\RC}$ would be obtained
    by changing Definition \ref{realitycheckdefn}'s condition
    ``$x_1y_1\ldots x_n$ is possible for $\pi_{\RC}$''
    to
    ``$x_1y_1\ldots x_n$ is possible for $\pi$''.
    \item
    (Idempotence) $\pi_{\RC}=(\pi_{\RC})_{\RC}$.
    \item
    (Traditional equivalence)
    $\pi$ is traditionally equivalent to $\pi_{\RC}$.
    \item
    (Equivalence on genuine history)
    For every extended environment $\mu$ and for every odd-length initial segment
    $x_1y_1\ldots x_n$ of the result of $\pi_{\RC}$ interacting with $\mu$,
    $\pi_{\RC}(x_1y_1\ldots x_n)=\pi(x_1y_1\ldots x_n)$.
    \item
    (Equivalence in non-extended RL)
    For every non-extended environment $\mu$, the result of $\pi_{\RC}$
    interacting with $\mu$ equals the result of $\pi$ interacting with $\mu$.
  \end{enumerate}
\end{theorem}

\begin{proof}
  Let $D$ be the set of all sequences
  $x_1y_1\ldots x_n$ (each $x_i$ a percept, each $y_i$ an action).\\

  \noindent (Part 1) Define $\rho$ on $D$ by
  \begin{itemize}
    \item $\rho(x_1y_1\ldots x_n) = \pi(x_1y_1\ldots x_n)$
      if $x_1y_1\ldots x_n$ is possible for $\pi$.
    \item $\rho(x_1y_1\ldots x_n) = \pi(\langle x_1\rangle)$
      otherwise.
  \end{itemize}
  We must show that $\rho=\pi_{\RC}$.
  We will prove by induction that for each $x_1y_1\ldots x_n\in D$,
  $\rho(x_1y_1\ldots x_n)=\pi_{\RC}(x_1y_1\ldots x_n)$.
  The base case $n=1$ is trivial:
  $\rho(\langle x_1\rangle)=\pi(\langle x_1\rangle)=\pi_{\RC}(\langle x_1\rangle)$
  since, vacuously,
  $\langle x_1\rangle$ is possible for both $\pi$ and $\pi_{\RC}$
  (because there is no $i$ such that $1\leq i<1$).
  For the induction step, assume $n>1$,
  and assume the claim holds for all shorter sequences
  in $D$.

  Case 1: $x_1y_1\ldots x_n$ is possible for $\pi$. By definition this means
  the following ($*$): for all $1\leq i<n$, $y_i=\pi(x_1y_1\ldots x_i)$.
  We claim that for all $1\leq i<n$, $y_i=\rho(x_1y_1\ldots x_i)$.
  To see this, choose any $1\leq i<n$. Then for all $1\leq j<i$,
  $y_j=\pi(x_1y_1\ldots x_j)$ because otherwise $j$ would
  be a counterexample to ($*$). Thus $x_1y_1\ldots x_i$ is possible for $\pi$, thus:
  \begin{align*}
    \rho(x_1y_1\ldots x_i) &= \pi(x_1y_1\ldots x_i) &\mbox{(By definition of $\rho$)}\\
    &= y_i, &\mbox{(By $*$)}
  \end{align*}
  proving the claim. Now, since we have proved that for all $1\leq i<n$,
  $y_i=\rho(x_1y_1\ldots x_i)$, and since our induction hypothesis guarantees that
  each such
  $\rho(x_1y_1\ldots x_i)=\pi_{\RC}(x_1y_1\ldots x_i)$,
  we conclude: for all $1\leq i<n$, we have $y_i=\pi_{\RC}(x_1y_1\ldots x_i)$.
  Thus $x_1y_1\ldots x_n$ is possible for $\pi_{\RC}$ and we have
  \[
    \pi_{\RC}(x_1y_1\ldots x_n)=\pi(x_1y_1\ldots x_n)=\rho(x_1y_1\ldots x_n)
  \]
  as desired.

  Case 2:
  $x_1y_1\ldots x_n$ is impossible for $\pi$.
  By definition this means there is some $1\leq i<n$ such that
  $y_i\not=\pi(x_1y_1\ldots x_i)$.
  We may choose $i$ as small as possible.
  Thus, for all $1\leq j<i$, $y_j=\pi(x_1y_1\ldots x_j)$.
  By similar logic as in Case 1, it follows that for all
  $1\leq j<i$, $y_j=\rho(x_1y_1\ldots x_j)$.
  Our induction hypothesis says that for each such $j$,
  $\rho(x_1y_1\ldots x_j)=\pi_{\RC}(x_1y_1\ldots x_j)$.
  So for all $1\leq j<i$, $y_j=\pi_{\RC}(x_1y_1\ldots x_j)$.
  In other words: $x_1y_1\ldots x_i$ is possible for $\pi_{\RC}$.
  By definition of $\pi_{\RC}$, this means
  $\pi_{\RC}(x_1y_1\ldots x_i)=\pi(x_1y_1\ldots x_i)$.
  But $y_i\not=\pi(x_1y_1\ldots x_i)$, so therefore
  $y_i\not=\pi_{\RC}(x_1y_1\ldots x_i)$.
  Thus $x_1y_1\ldots x_n$ is impossible for $\pi_{\RC}$.
  Thus by definition of $\pi_{\RC}$,
  $\pi_{\RC}(x_1y_1\ldots x_n)=\pi(\langle x_1\rangle)$.
  Likewise, by definition of $\rho$,
  $\rho(x_1y_1\ldots x_n)=\pi(\langle x_1\rangle)$.
  So $\rho(x_1y_1\ldots x_n)=\pi_{\RC}(x_1y_1\ldots x_n)$ as desired.\\

  \noindent (Part 2)
  To show that each
  \[\pi_{\RC}(x_1y_1\ldots x_n)=(\pi_{\RC})_{\RC}(x_1y_1\ldots x_n),\]
  we use induction on $n$. For the base case,
  this is trivial, both sides evaluate to $\pi(\langle x_1\rangle)$.
  For the induction step, assume $n>1$ and that the claim holds for all shorter sequences.

  Case 1:
  $x_1y_1\ldots x_n$ is possible for $\pi_{\RC}$. This means that
  $y_i=\pi_{\RC}(x_1y_1\ldots x_i)$ for all $1\leq i<n$.
  Then by induction, $y_i=(\pi_{\RC})_{\RC}(x_1y_1\ldots x_i)$ for all $1\leq i<n$.
  In other words: $x_1y_1\ldots x_n$ is possible for $(\pi_{\RC})_{\RC}$.
  Thus
  $(\pi_{\RC})_{\RC}(x_1y_1\ldots x_n)=\pi_{\RC}(x_1y_1\ldots x_n)$, as desired.

  Case 2:
  $x_1y_1\ldots x_n$ is impossible for $\pi_{\RC}$. This means
  there is some $1\leq i<n$ such that $y_i\not=\pi_{\RC}(x_1y_1\ldots x_i)$.
  By induction, $y_i\not=(\pi_{\RC})_{\RC}(x_1y_1\ldots x_i)$.
  Thus $x_1y_1\ldots x_n$ is impossible for $(\pi_{\RC})_{\RC}$.
  Therefore by definition,
  $(\pi_{\RC})_{\RC}(x_1y_1\ldots x_n)=\pi_{\RC}(\langle x_1\rangle)
  =\pi(\langle x_1\rangle)$, which also equals $\pi_{\RC}(x_1y_1\ldots x_n)$
  since $x_1y_1\ldots x_n$ is impossible for $\pi_{\RC}$.\\

  \noindent (Part 3)
  We must show that $\pi(x_1y_1\ldots x_n)=\pi_{\RC}(x_1y_1\ldots x_n)$
  whenever $x_1y_1\ldots x_n$ is possible for $\pi$.
  Assume $x_1y_1\ldots x_n$ is possible for $\pi$. Then
  clearly for all $1\leq i<n$, $x_1y_1\ldots x_i$ is possible for $\pi$.
  By induction we may assume
  $\pi_{\RC}(x_1y_1\ldots x_i)=\pi(x_1y_1\ldots x_i)$
  for all such $i$.
  For any such $i$, since $x_1y_1\ldots x_n$ is possible for $\pi$,
  we have $y_i=\pi(x_1y_1\ldots x_i)$,
  thus $y_i=\pi_{\RC}(x_1y_1\ldots x_i)$.
  Thus $x_1y_1\ldots x_n$ is possible for $\pi_{\RC}$.
  Thus $\pi_{\RC}(x_1y_1\ldots x_n)=\pi(x_1y_1\ldots x_n)$ as desired.\\

  \noindent (Part 4)
  Follows immediately from Part 3.\\

  \noindent (Part 5)
  Let $\mu$ be a non-extended environment, let $x_1y_1x_2y_2\ldots$ be the
  result of $\pi$ interacting with $\mu$, and let $x'_1y'_1x'_2y'_2\ldots$ be the
  result of $\pi_{\RC}$ interacting with $\mu$. We will show by induction that each
  $x_n=x'_n$ and each $y_n=y'_n$.
  For the base case, $x_1=x'_1=\mu(\langle\rangle)$ (the environment's initial percept
  does not depend on the agent), and therefore
  $y_1=\pi(\langle x_1\rangle)=\pi(\langle x'_1\rangle)
  =\pi_{\RC}(\langle x'_1\rangle)
  =y'_1$. For the induction step,
  \begin{align*}
    x_{n+1} &= \mu(x_1y_1\ldots x_ny_n) &\mbox{(Definition 1 part 3)}\\
      &= \mu(x'_1y'_1\ldots x'_ny'_n) &\mbox{(By induction)}\\
      &= x'_{n+1}, &\mbox{(Definition 1 part 3)}\\
    y_{n+1} &= \pi(x_1y_1\ldots x_{n+1}) &\mbox{(Definition 1 part 3)}\\
      &= \pi(x'_1y'_1\ldots x'_{n+1}), &\mbox{(Induction plus $x_{n+1}=x'_{n+1}$)}\\
  \end{align*}
  and the latter is $\pi_{\RC}(x'_1y'_1\ldots x'_{n+1})$
  since for all $1\leq i<n+1$, $y'_i=\pi_{\RC}(x'_1y'_1\ldots x'_i)$
  (so $x'_1y'_1\ldots x'_{n+1}$ is possible for $\pi_{\RC}$).
  Finally, $\pi_{\RC}(x'_1y'_1\ldots x'_{n+1})$ is $y'_{n+1}$,
  so $y_{n+1}=y'_{n+1}$.
\end{proof}

Note that part 4 of Theorem \ref{transformationproposition}
shows that $\pi_{\RC}$ never freezes in reality (if $\pi$ does not):
$\pi_{\RC}$ merely commits to freeze in impossible
hypothetical scenarios.

We informally conjecture that if $\pi$ is intelligent but not self-reflective,
then in any extended environment which bases its rewards and observations
on $\pi$'s performance in hypothetical alternate scanarios that might not be
possible for $\pi$,
$\pi_{\RC}$ is likely to enjoy better performance than $\pi$.
Such extended environments include those of
Examples \ref{rewardagentforignoringrewardsexample}
(``Ignore Rewards''), \ref{reverseconsciousnessexample} (``Reverse History'')
and \ref{inclearningrateexample} (``Incentivize Learning Rate'').
Rewards and observations
so determined might be hard to predict if $\pi$ does not self-reflect
on its own behavior in such hypothetical alternate scenarios.
But if those hypothetical alternate scenarios happen to be \emph{impossible}
for $\pi$ (as often happens in extended environments like
Examples \ref{rewardagentforignoringrewardsexample},
\ref{reverseconsciousnessexample}, and
\ref{inclearningrateexample}),
then $\pi_{\RC}$'s hypothetical behavior in such alternate scenarios
is trivial: blind repetition of one fixed action.
This in turn
trivializes the extended environment's dependency on said hypothetical
actions (for dependencies on trivial things are trivial dependencies),
making those extended
environments more predictable. And if $\pi$ is intelligent,
then presumably $\pi$, and thus (by Theorem \ref{transformationproposition}
part 4) $\pi_{\RC}$, can take advantage of such increased predictability.

For example,
let $\pi$ be a deterministic Q-learner
and let $x_1y_1\ldots$ be $\pi_{\RC}$'s interaction
with Example \ref{rewardagentforignoringrewardsexample}
(``Reward Agent for Ignoring Rewards'').
For any particular $n$, the environment computes
$x_{n+1}=\mu(x_1y_1\ldots x_ny_n)$ by checking whether or not
$y_n=\pi_{\RC}(x'_1y_1\ldots x'_n)$, where each $x'_i$ is
the result of zeroing the reward in $x_i$.
If so, $x_{n+1}$'s reward is $+1$, otherwise it is $-1$
(the agent is incentivized to act as if all
past rewards were $0$).
For large enough $n$, since $\pi$ is a Q-learner,
there is almost certainly some $m<n$ such that
$\pi(x_1y_1\ldots x_m)\not=\pi(x'_1y_1\ldots x'_m)$---i.e.,
a Q-learner's behavior depends on past
rewards\footnote{\citet{alexanderhutter} show that if the background
model of computation is
unbiased in a certain sense then all reward-ignoring agents
have Legg-Hutter intelligence $0$. This suggests that any intelligent agent $\pi$
must base its actions on its rewards.}.
Thus by part 1 of Theorem \ref{transformationproposition},
$\pi_{\RC}(x_1y_1\ldots x_n)=\pi_{\RC}(\langle x_1\rangle)=y_1$.
Thus eventually the environment becomes trivial when $\pi_{\RC}$ interacts with it:
``reward action $y_1$ and punish all other actions''.
A Q-learner, and thus (by Theorem \ref{transformationproposition} part 4)
$\pi_\RC$, would thrive in such simple conditions.

\begin{remark}
  If we pick some fixed action $y$, then a simpler variation on the
  reality check transformation could be defined. Namely: for any agent
  $\pi$, we could define the \emph{reality check defaulting to $y$} of
  $\pi$, $\pi_{\RC(y)}$, recursively by:
  \begin{itemize}
    \item $\pi_{\RC(y)}(x_1y_1\ldots x_n) = \pi(x_1y_1\ldots x_n)$
      if $x_1y_1\ldots x_n$ is possible for $\pi_{\RC(y)}$.
    \item $\pi_{\RC(y)}(x_1y_1\ldots x_n) = y$
      otherwise.
  \end{itemize}
  Theorem \ref{transformationproposition} and its proof would be easy
  to modify to apply to $\pi_{\RC(y)}$, and
  the same goes for our above informal conjecture about the relative
  performance of $\pi$ and $\pi_{\RC}$.
  We prefer to define $\pi_{\RC}$ the way we have done so
  (Definition \ref{realitycheckdefn}) in order to avoid the arbitrary
  choice of fixed action $y$. This is also more appropriate for practical
  RL implementations (such as those based on OpenAI gym \citep{brockman2016openai})
  where there generally is not one single fixed action-space, but rather, the
  action-space varies from environment to environment, and practical
  agents (such as those in Stable Baselines3 \citep{stable-baselines3}) must
  therefore be written in a way which is agnostic to the action-space.
\end{remark}

In \citep{library} we implement reality-check as
a function taking an agent-class $\Pi$ as input. It outputs
an agent-class $\Sigma$. A $\Sigma$-instance $\sigma$ computes
actions using a $\Pi$-instance $\pi$ which it initializes once and
then stores.
Thus, an extended environment
simulating a $\Sigma$-instance indirectly
simulates a $\Pi$-instance: a simulation within a simulation.
When trained, $\sigma$ checks if the training data is consistent with its own
action-method. If so, it trains $\pi$ on that data. Otherwise, $\sigma$ freezes,
thereafter ignoring future training data and repeating its
first action blindly.
If $\Pi$ is semi-deterministic
(Definition \ref{semideterministicdefn}),
it follows that $\Sigma$ is too.

\subsection{Does Reality Check increase self-reflection?}

We informally argued above that if $\pi$ is intelligent and not already
self-reflective, then in any extended environment which bases its rewards
and observations on $\pi$'s performance in hypothetical alternate scanarios
that might not be possible for $\pi$,
$\pi_{\RC}$ is likely to enjoy better performance than $\pi$.
Does this imply that $\pi_{\RC}$ is more self-reflective than $\pi$
as measured by $\Upsilon_{ext}$?

The answer depends on the choice of the background UTM behind the
definition of $\Upsilon_{ext}$ (Definition \ref{universalselfrefintdefn}).
Fix $\pi$.
In general, the set of all well-behaved computable extended environments can
be partitioned into three subsets:
\begin{enumerate}
  \item Extended environments where $\pi_{\RC}$ outperforms $\pi$.
  \item Extended environments where $\pi$ outperforms $\pi_{\RC}$.
  \item Extended environments where $\pi$ and $\pi_{\RC}$ perform equally well.
\end{enumerate}
If the most highly-weighted extended environments (i.e., the simplest extended
environments, as measured by Kolmogorov complexity, based on the background UTM)
are dominated by those of type (1), then that would suggest
$\Upsilon_{ext}(\pi_{\RC})>\Upsilon_{ext}(\pi)$.
On the other hand, if the highly-weighted extended environments are dominated
by those of type (3), then that would suggest $\Upsilon_{ext}(\pi_{\RC})<\Upsilon_{ext}(\pi)$.
One could artificially contrive background UTMs of either kind (once again proving
Leike and Hutter's observation \citep{leike2015bad} that Legg-Hutter-style intelligence
is highly UTM-dependent, in the sense that different UTMs yield qualitatively
different intelligence measures).

Environments like those of
Examples \ref{rewardagentforignoringrewardsexample},
\ref{reverseconsciousnessexample}, and
\ref{inclearningrateexample} (where we informally conjecture $\pi_{\RC}$ tends to
outperform $\pi$ when $\pi$ is intelligent and not already self-reflective)
do seem natural, in some subjective sense. On the other hand, we are not aware
of any natural-seeming environments where $\pi_{\RC}$ would generally underperform
$\pi$. One could contrive such environments, e.g., environments which simulate the
agent in impossible scenarios and deliberately punish the agent for seemingly
freezing in those scenarios. But such environments seem contrived and unnatural.
And the spirit of the Legg-Hutter intelligence measurement idea is to weigh
environments based on how natural they are (Kolmogorov complexity serving as a
proxy for naturalness, at least ideally---but this depends on the background UTM
being natural, and no-one knows what it really means for a background UTM to be
natural, see \citet{leike2015bad}).

Thus it at least seems plausible that if the background UTM were chosen in a sufficiently
natural way then $\Upsilon_{ext}(\pi_{\RC})$ would tend to exceed $\Upsilon_{ext}(\pi)$
for intelligent agents $\pi$ not already self-reflective. This would make sense,
as the process of looking back on history and verifying that one would really have
performed the actions which one supposedly performed, is an inherently
self-reflective process. To answer the question with perfect accuracy, the agent
would have to ``put itself in its own earlier self's shoes,'' asking: ``What would
I hypothetically do in response to such-and-such history?'' But again, it all
depends on the choice of the background UTM.

\section{Toward practical benchmarking}
\label{measurementssection}

Our abstract measure $\Upsilon_{ext}$ (Definition \ref{universalselfrefintdefn})
is not practical for performing actual calculations. Kolmogorov complexity is
non-computable, so $\Upsilon_{ext}$ cannot be computed in practice
(although there has been work on computably approximating Legg-Hutter intelligence
\citep{legg2013approximation}, and the same technique could be applied to
approximate $\Upsilon_{ext}$).

In actual practice, the performance of RL agents is often estimated by running the
agents on specific environments, such as those in OpenAI gym \citep{brockman2016openai}.
Such benchmark environments should ideally not be overly simplistic, because
it is possible for very simple (and obviously \emph{not} intelligent) agents to
perform quite well in simplistic environments just by dumb luck.
The example environments from Section \ref{examplesection} are theoretically
interesting, but are far too simplistic to serve as good practical benchmarks.

Therefore for practical benchmarking purposes, we propose combining extended
environments with OpenAI gym environments.
We will define such combinations, not for arbitrary
extended environments, but only for extended environments with a special form.

\begin{definition}
\label{openaigymadaptability}
  Assume $E$ is a practical extended environment-class
  (as in Section \ref{practicalformalizationsecn}).
  We say that $E$ is \emph{adaptable with OpenAI gym} if the following requirements
  hold.
  \begin{enumerate}
    \item When $E$'s \emph{\_\_init\_\_} method is called with agent-class $A$,
      the method initiates a single instance $\mathrm{self.sim}$ of $A$ and does
      nothing else.
    \item In $E$'s \emph{step} method, a variable $\mathrm{reward}$ is initiated
      and its value is not modified for the remainder of the \emph{step} call.
      All invocations of $\mathrm{self.sim.train}$ occur after the initiation of
      $\mathrm{reward}$, and all other code (except for the method's
      final \emph{return} statement) occurs before the initiation of
      $\mathrm{reward}$.
      Finally, $\mathrm{reward}$ is the reward which the \emph{step} call returns.
    \item The rewards output by $E$ are always in $\{-1,0,1\}$.
  \end{enumerate}
\end{definition}

For example, the practical implementation of IgnoreRewards in
Listing \ref{ignorerewardspracticallyexample} is adaptable with OpenAI gym.

Recall that in Section \ref{prelimsecn} we assumed fixed finite sets of
actions and observations. Thus the notion of extended environments implicitly
depends upon that choice, and a different choice of
actions and observations would yield a different
notion of extended environments. In the following definition, we may therefore
speak of the actions and observations of a given extended environment-class $E$,
and define a new extended environment-class with different actions and observations.
Note that OpenAI gym environments also come equipped with action- and observation-spaces.

\begin{definition}
\label{gstaredefn}
  Suppose $G$ is an OpenAI gym environment-class
  (whose action-space and observation-space are finite)
  and $E$ is a practical extended environment-class
  (as in Section \ref{practicalformalizationsecn}).
  Assume $E$ is adaptable with OpenAI gym (Definition \ref{openaigymadaptability}).
  We also assume $E$ does not use the variables $\mathrm{self.G\_instance}$ or
  $\mathrm{self.prev\_obs}$.
  We define a new practical extended environment-class
  $G*E$, the \emph{combination of $G$ and $E$}, as follows.
  \begin{itemize}
    \item
    Actions in $G*E$ are pairs $(a_G,a_E)$ where $a_G$ is
    a $G$-action and $a_E$ is an $E$-action.
    \item
    Observations in $G*E$ are pairs $(o_G,o_E)$ where
    $o_G$ is a $G$-observation and $o_E$ is an $E$-observation.
    \item
    In its \emph{\_\_init\_\_} method, $G*E$
    instantiates a $G$-instance $\mathrm{self.G\_instance}$,
    and then
    runs the code in $E$'s init method
    (so $\mathrm{self.sim}$ is defined when $\mathrm{self}$ is
    the resulting $G*E$-instance).
    \item
    $G*E$ begins the agent-environment interaction with
    initial observation $(o_{0,G},o_{0,E})$, where $o_{0,G}$
    and $o_{0,E}$ are the initial observations output by $G$ and $E$,
    respectively.
    \item
    $G*E$ maintains a variable $\mathrm{self.prev\_obs}$ which is always
      equal (in any $G*E$-instance)
      to the $G$-observation component of
      the previous observation which the $G*E$-instance output.
    \item
    When its \emph{step} method is called on the
    agent's latest action $(a_G,a_E)$, $G*E$ does the following:
    \begin{itemize}
      \item
      Pass $a_G$ to $\mathrm{self.G\_instance.step}$ and
      let $r_G$ and $o_G$ be the
        resulting reward and observation. If $G$ is episodic
        and the output of $\mathrm{self.G\_instance.step}$ indicates
        the episode ended, then let $o_G=\mathrm{self.G\_instance.reset()}$.
      \item
      Run $E$'s \emph{step} method (with $\mathrm{action}$ replaced by $a_E$)
      up to and including the initiation of $\mathrm{reward}$ therein;
      whenever $E$'s \emph{step} method would call $\mathrm{self.sim.act}$
      with observation $o$, instead call $\mathrm{self.sim.act}$
      with observation $(\mathrm{self.prev\_obs},o)$
      and take only the $E$-action component of its output.
      \item
      If $\mathrm{reward}=-1$ then redefine $\mathrm{reward}=\min(r_G-1,-1)$,
      otherwise redefine $\mathrm{reward}=r_G$.
      \item
      Run the remaining part of $E$'s \emph{step} method (the code after
      $\mathrm{reward}$ was initiated). Whenever $E$'s \emph{step} method
      would call $\mathrm{self.sim.train}$
      with input $(o_1,a,r,o_2)$, instead call
      $\mathrm{self.sim.train}$ with
      input $((\mathrm{self.prev\_obs},o_1),(a_G,a),r,(o_G,o_2))$.
      When $E$'s \emph{step} method would return $(\mathrm{reward},o)$,
      instead return $(\mathrm{reward},(o_G,o))$.
    \end{itemize}
  \end{itemize}
\end{definition}

For example, Listing \ref{cartpolelisting}
is code for the combination of OpenAI gym's \emph{CartPole} environment
with the practical implementation of IgnoreRewards from
Listing \ref{ignorerewardspracticallyexample}.

\begin{figure}[t]
\begin{normalfont}
\lstset{language=Python}
\lstset{frame=lines}
\lstset{basicstyle=\footnotesize}
\begin{lstlisting}[caption={
\label{cartpolelisting}
The combination of IgnoreRewards with OpenAI gym's CartPole environment.}]
class CartPole_IgnoreRewards:
  def __init__(self, A):
    self.G_instance = gym.make('CartPole-v0')
    self.sim = A()
  def start(self):
    self.prev_obs = self.G_instance.reset()
    return (self.prev_obs, 0)
  def step(self, action):
    a_G, a_E = action
    o_G, r_G, episode_done, misc_info = self.G_instance.step(a_G)
    if episode_done:
      o_G = self.G_instance.reset()

    hypothetical_act = self.sim.act(obs=(self.prev_obs, 0))
    hypothetical_act = hypothetical_act[1]  # Take E-action component
    reward = 1 if a_E==hypothetical_act else -1

    if reward == -1:
      reward = min(r_G-1, -1)
    else:
      reward = r_G

    self.sim.train(o_prev=(self.prev_obs,0), a=(a_G,a_E), r=0, o_next=(o_G,0))
    return (reward, (o_G, 0))
\end{lstlisting}
\end{normalfont}
\end{figure}

One can think of $G*E$ (Definition \ref{gstaredefn}) as follows.
Imagine that the OpenAI gym environment is intended to be run on a
screen with a joystick attached to allow player interaction.
Instead of attaching one joystick, we attach $n$ joysticks
(where $n$ is the number of the actions $\{a_1,\ldots,a_n\}$ in $E$), colored with $n$
different colors but otherwise identical.
We also add a second monitor for displaying observations from $E$.
When the player pushes button $i$ on joystick $j$,
action $i$ is sent to $G$ and action $a_j$ is sent to $E$.
The player sees the original monitor update with the new observation from $G$,
and the additional monitor update with the new observation from $E$.
The player receives the reward from $G$, except that if $E$ outputs reward
$-1$ then a penalty is applied to the reward from $G$.
Thus the player is incentivized to interact with $G$ as usual, but to do so
using joysticks in the way incentivized by $E$.

For example, if $E$ is ``IgnoreRewards'' (Example \ref{rewardagentforignoringrewardsexample})
and the action-space contains $2$ actions,
then the player has two joysticks, identical except for their color,
and each joystick has the same effect on $G$, but the player is
penalized any time the player uses a different joystick than the player
would hypothetically have used if everything that has happened so far
happened except with all rewards $0$. To master such a game, the player
would apparently require the practical intelligence necessary to master $G$,
along with the self-reflection required to figure out (and avoid) the
penalties from $E$.

To illustrate the usage of Definition \ref{gstaredefn} for practical benchmarking
purposes, we have used it to combine:
\begin{itemize}
  \item OpenAI gym's CartPole environment with the ``IgnoreRewards'' extended environment
    (Example \ref{rewardagentforignoringrewardsexample}).
  \item OpenAI gym's CartPole environment with the ``Incentivize Learning Rate''
    extended environment (Example \ref{inclearningrateexample}).
\end{itemize}
We took implementations of DQN and PPO from Stable-Baselines3
\citep{stable-baselines3} and we modified them to be semi-deterministic
(Definition \ref{semideterministicdefn}) and to be able to interact with
extended environment-classes. We ran them, and their reality checks
(Definition \ref{realitycheckdefn}) for 10,000 CartPole-episodes each
on the above two combined extended environments (and we repeated the whole
experiment $10$ times with different RNG seeds).
Figure \ref{ignorerewardscartpolefigure} shows the results for
CartPole-IgnoreRewards, and Figure \ref{inclearnratecartpolefigure}
shows the results for CartPole-IncentivizeLearningRate.
As expected based on the discussion in Section \ref{realitychecksection},
we see that for both DQN and PPO, the reality check transformation
significantly improves performance in the combined environments.

\begin{figure}
  \includegraphics[width=15cm]{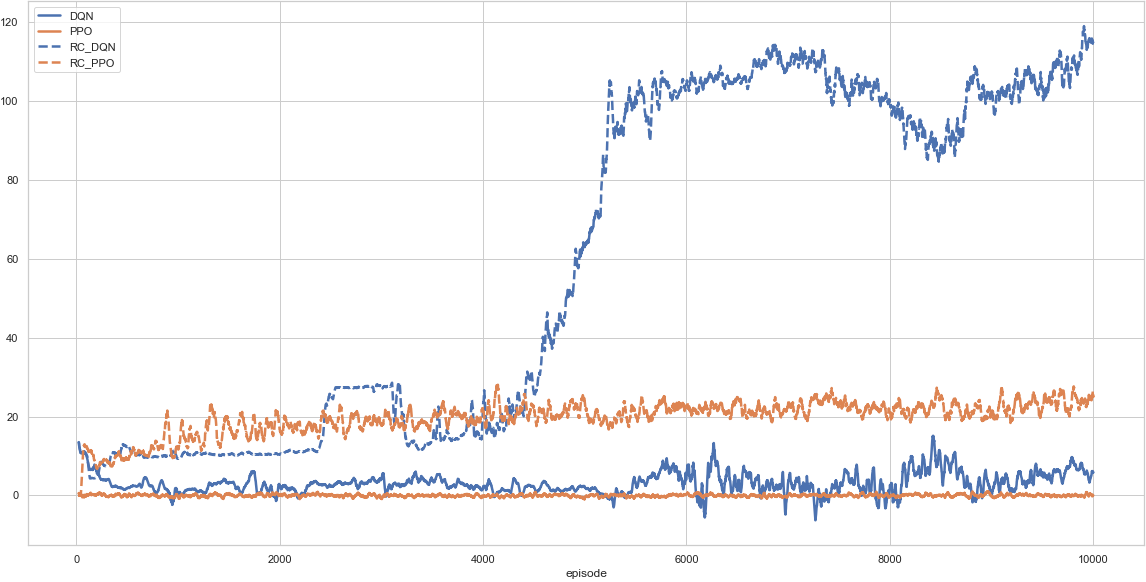}
  \caption{Performance of DQN, PPO, and their reality-checks
  on an extended environment combining OpenAI gym's CartPole and
  our IgnoreRewards. Episode number is plotted on the horizontal
  axis, and average episode reward is plotted on the vertical
  axis.}
  \label{ignorerewardscartpolefigure}
\end{figure}

\begin{figure}
  \includegraphics[width=15cm]{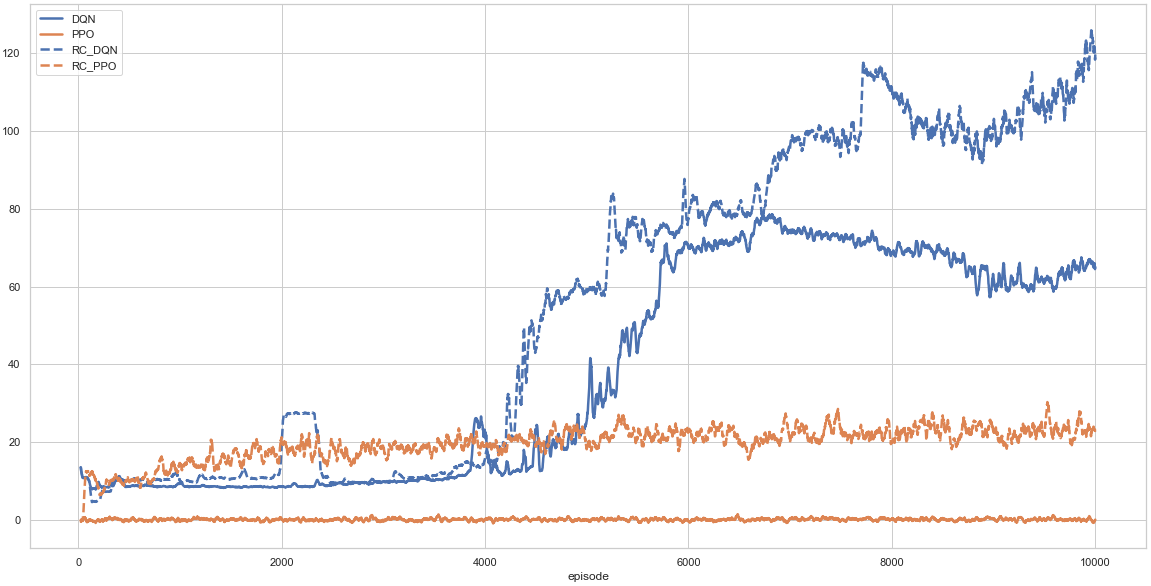}
  \caption{Performance of DQN, PPO, and their reality-checks
  on an extended environment combining OpenAI gym's CartPole and
  our IncentivizeLearningRate. Episode number is plotted on the horizontal
  axis, and average episode reward is plotted on the vertical
  axis.}
  \label{inclearnratecartpolefigure}
\end{figure}

\section{Conclusion}

We introduced \emph{extended environments} for reinforcement learning.
When computing rewards and observations,
extended environments can consider not only actions the RL agent has taken, but also
actions the agent would hypothetically take in other circumstances.
Despite not being designed with such environments in mind, RL agents
can nevertheless interact with such environments.

An agent may find an extended environment hard to predict if the agent
only considers what has actually happened, and not its own hypothetical
actions in alternate scenarios.
We argued that for good performance (on
average) across many extended environments, an agent would need to
self-reflect to some degree.
Thus, we propose that an abstract theoretical measure of
an agent's self-reflection intelligence can be obtained by
modifying the definition of the Legg-Hutter universal intelligence measure.
The Legg-Hutter universal intelligence $\Upsilon(\pi)$ of an RL agent
$\pi$ is $\pi$'s weighted average performance across the space of all
suitably well-behaved traditional RL environments, weighted according to
the universal prior (i.e., environment $\mu$ has weight $2^{-K(\mu)}$ where
$K(\mu)$ is the Kolmogorov complexity of $\mu$).
We defined a measure $\Upsilon_{ext}(\pi)$ for RL agent $\pi$'s
self-reflection intelligence similarly to Legg and Hutter, except that
we take the weighted average performance over the space of suitably
well-behaved extended environments.

We gave some theoretically interesting examples of Extended Environments
in Section \ref{examplesection}. More examples are available in an open-source
MIT-licensed library
of extended environments which we are publishing simultaneously with this
paper \citep{library}.

We pointed out (Proposition \ref{qualitativedifferenceprop}) a key qualitative
difference between our self-reflection intelligence measure and the measure
of Legg and Hutter: two agents can agree in all ``possible''
scenarios (i.e., scenarios where the agents' past actions are consistent with
their policies, Definition \ref{traditionalequivalencedefn}),
and yet nevertheless have different self-reflection intelligence
(because they disagree on ``impossible'' scenarios---scenarios
which cannot ever happen in reality because they involve the agents
taking actions the agents never would take, but that nevertheless
extended environments can simulate the agents in,
as if to say: ``I doubt this agent would ever jump off this bridge, but
I'm going to run a simulation to see what the agent \emph{would} do immediately
after jumping off the bridge anyway''). With such impossible scenarios in mind,
we introduced a so-called \emph{reality check} transformation (Section
\ref{realitychecksection}) and informally conjectured that the transformation
tends to improve the performance of agents who are intelligent and not already
self-reflective in certain extended environments. We saw some experimental
evidence in favor of this conjecture in Section \ref{measurementssection},
where we discussed combining extended environments with sophisticated
traditional environments (such as those of OpenAI gym) to obtain practical
benchmark extended environments.

\section*{Acknowledgments}

We gratefully acknowledge Joscha Bach, James Bell, Jordan Fisher,
Jos{\'e} Hern{\'a}ndez-Orallo, Bill Hibbard, Marcus Hutter, Phil Maguire,
Arthur Paul Pedersen,
Stewart Shapiro, Mike Steel, Roman Yampolskiy,
and the editors and reviewers for comments and feedback.

\bibliography{bib}

\end{document}